\documentclass{article}

\PassOptionsToPackage{numbers, compress}{natbib}

\usepackage[preprint]{nips_2018}

\usepackage[utf8]{inputenc} 
\usepackage[T1]{fontenc}    
\usepackage{hyperref}       
\usepackage{url}            
\usepackage{booktabs}       
\usepackage{amsfonts}       
\usepackage{nicefrac}       
\usepackage{microtype}      

\usepackage{amsthm,newtxtext,newtxmath,amssymb,amsmath,amsfonts}
\usepackage{graphicx}
\usepackage{subcaption}
\usepackage{algorithm}
\usepackage{algcompatible}
\usepackage{xcolor}
\usepackage{hyperref}
\usepackage[bbgreekl]{mathbbol}

\newtheorem{definition}{Definition}
\newtheorem{assumption}{Assumption}
\newtheorem{claim}{Claim}

\newtheorem{corollary}{Corollary}
\newtheorem{theorem}{Theorem}
\newtheorem{proposition}{Proposition}

\newcommand{\R}{\mathbb{R}}
\newcommand{\E}{\mathbb{E}}
\newcommand{\cE}{\mathcal{E}}
\newcommand{\X}{\mathcal{X}}
\newcommand{\Y}{\mathcal{Y}}

\newcommand{\V}{\mathcal{V}}

\newcommand{\bx}{\boldsymbol{x}}
\newcommand{\by}{\boldsymbol{y}}
\newcommand{\bomega}{\boldsymbol{\omega}}

\newcommand{\bu}{\boldsymbol{u}}
\newcommand{\bv}{\boldsymbol{v}}
\newcommand{\bw}{\boldsymbol{w}}

\newcommand{\bphi}{\boldsymbol{\phi}}

\newcommand{\cH}{\mathcal{H}}

\newcommand{\hf}{\hat{f}}
\newcommand{\tf}{\tilde{f}}
\newcommand{\tk}{\tilde{k}}

\newcommand{\balpha}{\boldsymbol{\alpha}}
\newcommand{\hL}{\hat{L}}
\newcommand{\talpha}{\tilde{\alpha}}
\newcommand{\softmin}{\mathit{softmin}}

\title{D2KE: From Distance to Kernel and Embedding}

\author{
   Lingfei Wu \thanks{Both authors contributed equally to this work}\\
   IBM Research\\
   Yorktown Heights, NY 10598\\
   \texttt{wuli@us.ibm.com} \\
   \And
   Ian En-Hsu Yen {\small{$^\ast$}}\\
  Carnegie Mellon University\\
  Pittsburgh, PA 15213 \\
  \texttt{eyan@cs.cmu.edu} \\
  \And
   Fangli Xu \\
  College of William and Mary\\
  Williamsburg, VA 23185 \\
  \texttt{fxu02@email.wm.edu} \\
  \AND
   Pradeep Ravikumar \\
  Carnegie Mellon University\\
  Pittsburgh, PA 15213 \\
  \texttt{pradeepr@cs.cmu.edu} \\
   \And
   Michael Witbrock\\
   IBM Research\\
   Yorktown Heights, NY 10598\\
   \texttt{witbrock@us.ibm.com} \\
}

\begin{document}

\maketitle

\begin{abstract}
For many machine learning problem settings, particularly with structured inputs such as sequences or sets of objects, a distance measure between inputs can be specified more naturally than a feature representation. However, most standard machine models are designed for inputs with a vector feature representation. In this work, we consider the estimation of a function $f:\mathcal{X} \rightarrow \R$ based solely on a dissimilarity measure $d:\mathcal{X}\times\mathcal{X} \rightarrow \R$ between inputs. In particular, we propose a general framework to derive a family of \emph{positive definite kernels} from a given dissimilarity measure, which subsumes the widely-used \emph{representative-set method} as a special case, and relates to the well-known \emph{distance substitution kernel} in a limiting case. We show that functions in the corresponding Reproducing Kernel Hilbert Space (RKHS) are Lipschitz-continuous w.r.t. the given distance metric. We provide a tractable algorithm to estimate a function from this RKHS, and show that it enjoys better generalizability than Nearest-Neighbor estimates. Our approach draws from the literature of Random Features, but instead of deriving feature maps from an existing kernel, we construct novel kernels from a random feature map, that we specify given the distance measure. We conduct classification experiments with such disparate domains as strings, time series, and sets of vectors, where our proposed framework compares favorably to existing distance-based learning methods such as $k$-nearest-neighbors, distance-substitution kernels, pseudo-Euclidean embedding, and the representative-set method.
\end{abstract}

\section{Introduction}

In many problem domains, it is easier to specify a reasonable dissimilarity (or similarity) function between instances, than to construct a feature representation. This is particularly the case with structured inputs, such as sets, sequences, or networks of objects, where it is typically less than clear how to construct the representation of the whole structured input, even when given a good feature representation of each individual object. On the other hand, even for complex structured inputs, there are many well-developed dissimilarity measures, such as the Edit Distance (Levenshtein distance) between sequences, Dynamic Time Warping measure between time series, Hausdorff distance between sets, and Wasserstein distance between distributions. 

However, standard machine learning methods are designed for vector representations, and classically there has been far less work on distance-based methods for either classification or regression. The most common distance-based method is Nearest-Neighbor Estimation (NNE), which predicts the outcome for an instance using an average of its nearest neighbors in the input space, with nearness measured by the given dissimilarity measure. Estimation from nearest neighbors, however, is unreliable, specifically having high variance when the neighbors are far apart, which is typically the case when the intrinsic dimension implied by the distance is large.

To address this issue, a line of research has focused on developing global distance-based (or similarity-based) machine learning methods \cite{pkkalska2005dissimilarity,duin2012dissimilarity,balcan2008theory,cortes2012algorithms}, in large part by drawing upon connections to kernel methods \cite{scholkopf1999input} or directly learning with similarity functions \cite{balcan2008theory,cortes2012algorithms,balcan2008discriminative}; we refer the reader in particular to the survey in \cite{chen2009similarity}. Among these, the most direct approach treats the data similarity matrix (or a transformed dissimilarity matrix) as a kernel Gram matrix, and then uses standard kernel-based methods such as Support Vector Machines (SVM) or kernel ridge regression with this Gram matrix. A key caveat with this approach however is that most similarity (or dissimilarity) measures do not provide a \emph{positive-definite (PD)} kernel, so that the empirical risk minimization problem is not well-defined, and moreover becomes non-convex~\cite{ong2004learning,lin2003study}. 

A line of work has therefore focused on estimating a positive-definite (PD) Gram matrix that merely approximates the similarity matrix. This could be achieved for instance by clipping, or flipping, or shifting eigenvalues of the similarity matrix \cite{pekalska2001generalized}, or explicitly learning a PD approximation of the similarity matrix~\cite{chen2008training,chen2009learning}. Equivalently, one could also find a Euclidean embedding (also known as  dissimilarity representation) approximating the dissimilarity matrix as in Multidimensional Scaling~\cite{pekalska2001generalized,pkkalska2005dissimilarity,pekalska2006dissimilarity,pekalska2008beyond,duin2012dissimilarity}
Such modifications of the similarity matrix however often leads to a loss of information, and moreover, the enforced PD property is typically guaranteed to hold only on the training data, resulting in an inconsistency between the set of test and training samples~\cite{chen2009similarity} 
\cite{haasdonk2004learning,scholkopf2001kernel} provide conditions under which one can obtain a PD kernel through simple transformations of the distance measure, but which are not satisfied for many commonly used dissimilarity measures such as Dynamic Time Warping, Hausdorff distance and Earth Mover's distance~\cite{haasdonk2004learning}.

Another common approach is to select a subset of training samples as a held-out representative set, and use distances or similarities to points in the set as the feature function~\cite{graepel1999classification,pekalska2001generalized}. As we show, with proper scaling, this approach can be interpreted as a special instance of our framework. On the other hand, our framework provides a more general and richer family of kernels, many of which significantly outperform the representative-set method in a variety of application domains.

In this paper, we propose a general framework that constructs a family of PD kernels from a dissimilarity measure. The kernel satisfies the property that functions in the corresponding Reproducing Kernel Hilbert Space (RKHS) are Lipschitz-continuous w.r.t. the given distance measure. We also provide a tractable estimator for a function from this RKHS which enjoys much better generalization properties than nearest-neighbor estimation~\cite{chen2009similarity}. Our approach draws from the literature of Random Features \cite{rahimi2008random}, but instead of deriving feature maps from an existing kernel, we derive novel kernels from a random feature map specifically designed given the distance measure. Our framework produces a feature embedding  and consequently a vector representation of each instance that can be employed by any classification and regression models. In classification experiments in disparate domains as strings, time series, and sets of vectors, our proposed framework compares favorably to existing distance-based algorithms such as $k$-nearest-neighbors, distance-substitution kernels, pseudo-Euclidean embedding, and representative-set method.

\section{Problem Setup}

We consider the estimation of a target function $f:\X \rightarrow \R$ from a collection of samples $\{(\bx_i,y_i)\}_{i=1}^n$, where $\bx_i\in\X$ is the input object, and $y_i\in\Y$ is the output observation associated with the target function $f(\bx_i)$. For instance, in a regression problem, $y_i\sim f(\bx_i)+\omega_i \in \R$ for some random noise $\omega_i$, and in binary classification, we have $y_i\in\{0,1\}$ with $P(y_i=1|\bx_i)=f(\bx_i)$. We are given a dissimilarity measure $d:\X\times \X\rightarrow \R$ between input objects instead of a feature representation of $\bx$. For some of the analyses, we would require the dissimilarity measure to be a \emph{metric} as follows.

\begin{assumption}[Distance Metric]\label{assume:metric}
$d:\X\times \X\rightarrow \R$ is a distance metric, that is, it satisfies (i) $d(\bx_1,\bx_2)\geq 0$, (ii) $d(\bx_1,\bx_2)=0 \Leftarrow \Rightarrow \bx_1=\bx_2$, (iii) $d(\bx_1,\bx_2)=d(\bx_2,\bx_1)$, and (iv) $d(\bx_1,\bx_2)\leq d(\bx_1,\bx_3) + d(\bx_3,\bx_2)$.
\end{assumption}

\subsection{Function Continuity and Space Covering}

An ideal feature representation for the learning task is (i) compact and (ii) such that the target function $f(\bx)$ is a simple (e.g. linear) function of the resulting representation. Similarly, an ideal dissimilarity measure $d(\bx_1,\bx_2)$ for learning a target function $f(\bx)$ should satisfy certain properties. On the one hand, a small dissimilarity $d(\bx_1,\bx_2)$ between two objects should imply small difference in the function values $|f(\bx_1)-f(\bx_2)|$. On the other hand, we want a small expected distance among samples, so that the data lies in a compact space of small intrinsic dimension. We next build up some definitions to formalize these properties. 

\begin{assumption}[Lipschitz Continuity]\label{assume:lip}
For any $\bx_1, \bx_2 \in \X$, there exists some constant $L>0$ such that
\begin{equation}\label{lips}
|f(\bx_1)-f(\bx_2)| \leq L \, d(\bx_1,\bx_2),
\end{equation}
\end{assumption}

We would prefer the target function to have a small Lipschitz-continuity constant $L$ with respect to the dissimilarity measure $d(.,.)$. Such Lipschitz-continuity alone however might not suffice. For example, one can simply set $d(\bx_1,\bx_2)=\infty$ for any $\bx_1\neq \bx_2$ to satisfy Eq. \eqref{lips}. We thus need the following quantity that measures the size of the space implied by a given dissimilarity measure.

\begin{definition}[Covering Number]\label{def:cover}
Assuming $d$ is a \emph{metric}. A $\delta$-cover of $\X$ w.r.t. $d(.,.)$ is a set $\cE$ s.t. 
$$
\forall \bx\in\X, \exists \bx_i\in\cE, d(\bx,\bx_i)\leq \delta.
$$ 
Then the covering number $N(\delta; \X,d)$ is the size of the smallest $\delta$-cover for $\X$with respect to $d$.
\end{definition}

Assuming the input domain $\X$ is compact, the covering number $N(\delta; \X,d)$ measures its size w.r.t. the distance measure $d$. We show how the two quantities defined above affect the estimation error of a Nearest-Neighbor Estimator.

\subsection{Effective Dimension and Nearest Neighbor Estimation}

We extend the standard analysis of the estimation error of $k$-nearest-neighbor from finite-dimensional vector spaces to any input space $\X$, with an associated distance measure $d$, and a finite covering number $N(\delta; \X,d)$, by defining the \emph{effective dimension} as follows.

\begin{assumption}[Effective Dimension]\label{def:eff_dim}
Let the effective dimension $p_{\X,d}>0$ be the minimum $p$ satisfying
$$
\exists c>0, \forall \delta: 0< \delta < 1, \;\;N(\delta;\X,d)\leq c\left(\frac{1}{\delta}\right)^p.
$$
\end{assumption}

Here we provide an example of effective dimension in case of measuring the space of \emph{Multiset}.


\textbf{Multiset with Hausdorff Distance. } A multiset is a set that allows duplicate elements. Consider two multisets $\bx_1=\{\bu_i\}_{i=1}^M$, $\bx_2=\{\bv_j\}_{j=1}^N$. Let $\Delta(\bu_i,\bv_j)$ be a \emph{ground distance} that measures the distance between two elements $\bu_i,\bv_j\in\V$ in a set. The (modified) \emph{Hausdorff Distance} \cite{dubuisson1994modified} can be defined as $d(\bx_1,\bx_2):=$
\begin{equation}\label{HD}
\max\{\frac{1}{N}\sum_{i=1}^N \min_{j\in[M]}\Delta(\bu_i,\bv_j),\frac{1}{M}\sum_{j=1}^M\min_{i\in[N]}\Delta(\bv_j,\bu_i)\}
\end{equation}
Let $N(\delta;\V,\Delta)$ be the covering number of $\V$ under the ground distance $\Delta$. Let $\X$ denote the set of all sets of size bounded by $L$. By constructing a covering of $\X$ containing any set of size less or equal than $L$ with its elements taken from the covering of $\V$, we have
$
N(\delta;\X,d)\leq  N(\delta;\V;\Delta)^L.
$
Therefore,
$
p_{\X,d}\leq L\log N(\delta;\V,\Delta).
$
For example, if $\V:=\{\bv\in\R^{p}\mid \|\bv\|_2\leq 1\}$ and $\Delta$ is Euclidean distance, we have $N(\delta;\V,\Delta)=(1+\frac{2}{\delta})^p$ and 
$
p_{\X,d}\leq Lp.
$

Equipped with the concept of \emph{effective dimension}, we can obtain the following bound on the estimation error of the $k$-Nearest-Neighbor estimate of $f(\bx)$.

\begin{theorem}\label{thm:nn}
Let $Var(y|f(x))\leq \sigma^2$, and $\hat f_n$ be the $k$-Nearest Neighbor estimate of the target function $f$ constructed from a training set of size $n$. Denote $p:=p_{\X,d}$. We have
$$
\E_{\bx}\biggl[\left(\hat f_n(\bx)-f(\bx)\right)^2\biggr] \leq \frac{\sigma^2}{k} + cL^2\left(\frac{k}{n}\right)^{2/p}
$$
for some constant $c>0$. For $\sigma>0$, minimizing RHS w.r.t. the parameter $k$, we have
\begin{equation}\label{knn_est}
\E_{\bx}\biggl[\left(\hat f_n(\bx)-f(\bx)\right)^2\biggr] \leq c_2\sigma^{\frac{4}{p+2}}L^{\frac{2p}{2+p}} \left(\frac{1}{n}\right)^{\frac{2}{2+p}}
\end{equation}
for some constant $c_2>0$.
\end{theorem}
\begin{proof}
The proof is almost the same to a standard analysis of $k$-NN's estimation error in, for example, \cite{gyorfi2006distribution}, with the \emph{space partition number} replaced by the \emph{covering number}, and \emph{dimension} replaced by the \emph{effective dimension} in Defition \ref{def:eff_dim}.
\end{proof}

When $p_{\X,d}$ is reasonably large, the estimation error of $k$-NN decreases quite slowly with $n$. Thus, for the estimation error to be bounded by $\epsilon$, requires the number of samples to scale exponentially in $p_{\X,d}$. In the following sections, we develop an estimator $\hat{f}$ based on a RKHS derived from the distance measure, with a considerably better sample complexity for problems with higher effective dimension.

\section{From Distance to Kernel}
\label{sec:d2ke}
We aim to address the long-standing problem of how to convert a distance measure into a positive-definite kernel. Existing approaches either require strict conditions on the distance function (e.g. that the distance be isometric to the square of the Euclidean distance)~\cite{haasdonk2004learning,scholkopf2001kernel}, or construct empirical PD Gram matrices that do not necessarily generalize to the test samples \cite{pekalska2001generalized}. There are also some approaches specific to some structured inputs such as sequences, such as \cite{collins2002convolution,cuturi2011fast} that modify a distance function over sequences to a kernel by replacing the minimization over possible alignments into a summation over all possible alignments. This type of kernel, however, results in a diagonal-dominance problem, where the diagonal entries of the kernel Gram matrix are orders of magnitude larger than the off-diagonal entries, due to the summation over a huge number of alignments with a sample itself.

Here we introduce a simple but effective approach \emph{D2KE} that constructs a family of \emph{positive-definite} kernels from a given distance measure. Given an input domain $\X$ and a distance measure $d(.,.)$, we construct a family of kernels as
\begin{equation}\label{DKernel}
k(\bx,\by):=\int p(\bomega) \phi_{\bomega}(\bx)\phi_{\bomega}(\by) d\bomega, \text{where}\;\; \phi_{\bomega}(\bx):=\exp(-\gamma d(\bx,\bomega)),
\end{equation}
where $\bomega\in \Omega$ is a random structured object, $p(\bomega)$ is a distribution over $\Omega$, and $\phi_{\bomega}(\bx)$ is a feature map derived from the distance of $\bx$ to all objects $\bomega\in\Omega$. The kernel is parameterized by both $p(\bomega)$ and $\gamma$. 

\textbf{Relationship to Distance Substitution Kernel.} 
An insightful interpretation of the kernel \eqref{DKernel} can be obtained by expressing the kernel \eqref{DKernel} as
\begin{equation}\label{DKernel2}
\exp\left( -\gamma\softmin_{p(\bomega)}\{ d(\bx,\bomega)+d(\bomega,\by) \} \right)
\end{equation}
where the soft minimum function, parameterized by $p(\bomega)$ and $\gamma$, is defined as
\begin{equation}\label{softmin}
\softmin_{p(\bomega)}\;f(\bomega):= -\frac{1}{\gamma}\log \int p(\bomega) e^{-\gamma f(\bomega)} d\bomega.
\end{equation} 
Therefore, the kernel $k(\bx,\by)$ can be interpreted as a soft version of the \emph{distance substitution kernel} \cite{haasdonk2004learning}, where instead of substituting $d(\bx,\by)$ into the exponent, it substitutes a soft version of the form
\begin{equation}\label{soft_dist}
\softmin_{p(\omega)}\{ d(\bx,\bomega)+d(\bomega,\by) \}.
\end{equation}
Note when $\gamma\rightarrow\infty$, the value of \eqref{soft_dist} is determined by $\min_{\bomega\in\Omega}\;  d(\bx,\bomega)+d(\bomega,\by) $, which equals $d(\bx,\by)$ if $\X\subseteq\Omega$, since it cannot be smaller than $d(\bx,\by)$ by the triangle inequality. In other words, when $\X\subseteq \Omega$,
$$
k(\bx,\by) \rightarrow \exp(-\gamma d(\bx,\by)) \;\;\text{as}\;\; \gamma\rightarrow \infty.
$$
On the other hand, unlike the distance-substituion kernel, our kernel in Eq. \eqref{DKernel2} is always PD by construction.

\begin{algorithm}[t]
  \caption{Random Feature Approximation of function in RKHS with the kernel in  \eqref{DKernel}}
 \begin{algorithmic}[1] 
     \vspace{+0.1cm}
     \STATE Draw $R$ samples from $p(\bomega)$ to get $\{\bomega_j\}_{j=1}^R$.
     \vspace{+0.1cm}
     \STATE Set the $R$-dimensional feature embedding as 
     $$
     \hat\bphi_j(\bx)=\frac{1}{\sqrt{R}}\exp(-\gamma d(\bx,\bomega_j) ),\;\forall j\in[R]
     $$
     \STATE Solve the following problem for some $\mu>0$:
     $$
     \hat \bw :=\underset{\bw\in\R^{R}}{argmin}\; \frac{1}{n}\sum_{i=1}^n \ell(\bw^T\hat\bphi(\bx_i),y_i) + \frac{\mu}{2}\|\bw\|^2
     $$
     \STATE Output the estimated function $\tilde f_R(\bx):=\hat\bw^T\hat\bphi(\bx)$.
 \end{algorithmic}
 \label{alg:RF}
 \end{algorithm}

\textbf{Random Feature Approximation.}
The reader might have noticed that the kernel \eqref{DKernel} cannot be evaluated analytically in general. However, this does not prohibit its use in practice, so long as we can approximate it via \emph{Random Features (RF)}~\cite{rahimi2008random}, which in our case is particularly natural as the kernel itself is defined via a random feature map. Thus, our kernel with the RF approximation can not only be used in small problems but also in large-scale settings with a large number of samples, where standard kernel methods with $O(n^2)$ complexity are no longer efficient enough and approximation methods, such as Random Features, must be employed \cite{rahimi2008random,wu2016revisiting,chen2016efficient,kar2012random,bach2017equivalence}. Given the RF approximation, one can then directly learn a target function as a linear function of the RF feature map, by minimizing a domain-specific empirical risk.
It is worth noting that a recent work \cite{sinha2016learning} that learns to select a set of random features by solving an optimization problem in an supervised setting is orthogonal to our D2KE approach and could be extended to develop a supervised D2KE method.  
We outline this overall \emph{RF} based empirical risk minimization for our class of D2KE kernels in Algorithm~\ref{alg:RF}. We will provide a detailed analysis of our estimator in Algorithm~\ref{alg:RF} in Section~\ref{sec:analysis}, and contrast its statistical performance to that of $K$-nearest-neighbor.

\textbf{Relationship to Representative-Set Method.} 
A naive choice of $p(\bomega)$ relates our approach to the \emph{representative-set method (RSM)}: setting $\Omega=\X$, with $p(\bomega)=p(\bx)$. This gives us a kernel \eqref{DKernel} that depends on the data distribution. One can then obtain a  Random-Feature approximation to the kernel in \eqref{DKernel} by holding out a part of the training data $\{\hat\bx_j\}_{j=1}^R$ as samples from $p(\bomega)$, and creating an $R$-dimensional feature embedding of the form:
\begin{equation}\label{feature_embed}
\hat{\phi}_{j}(\bx):=\frac{1}{\sqrt{R}}\exp\left(-\gamma d(\bx,\hat\bx_j)\right),\; j\in[R],
\end{equation}
as in Algorithm \ref{alg:RF}. This is equivalent to a $1/\sqrt{R}$-scaled version of the embedding function in the \emph{representative-set method} (or \emph{similarity-as-features method}) \cite{graepel1999classification,pekalska2001generalized,pkkalska2005dissimilarity,pekalska2006dissimilarity,pekalska2008beyond,chen2009similarity,duin2012dissimilarity}, where one computes each sample's similarity to a set of representatives as its feature representation. However, here by interpreting \eqref{feature_embed} as a random-feature approximation to the kernel \eqref{DKernel}, we obtain a much nicer generalization error bound even in the case $R\rightarrow \infty$. This is in contrast to the analysis of RSM in \cite{chen2009similarity}, where one has to keep the size of the representative set small (of the order $o(n)$) in order to have reasonable generalization performance.

\textbf{Effect of $p(\bomega)$.}
The choice of $p(\bomega)$ plays an important role in our kernel. Surprisingly, we found that many ``close to uniform'' choices of $p(\bomega)$ in a variety of domains give better performance than for instance the choice of the data distribution $p(\bomega)=p(\bx)$ (as in the representative-set method). Here are some examples from our experiments: 
i) In the \emph{time-series} domain with dissimilarity computed via Dynamic Time Warping (DTW), a distribution $p(\bomega)$ corresponding to random time series of length uniform in $\in [2,10]$, and with Gaussian-distributed elements, yields much better performance than the Representative-Set Method (RSM);
ii) In \emph{string} classification, with edit distance, a distribution $p(\bomega)$ corresponding to random strings with elements uniformly drawn from the alphabet $\Sigma$ yields much better performance than RSM; 
iii) When classifying sets of vectors with the Hausdorff distance in Eq. \eqref{HD}, a distribution $p(\bomega)$ corresponding to random sets of size uniform in $\in [3,15]$ with elements drawn uniformly from a unit sphere  yields significantly better performance than RSM.

We conjecture two potential reasons for the better performance of the chosen distributions $p(\bomega)$ in these cases, though a formal theoretical treatment is an interesting subject we defer to future work.
Firstly, as $p(\bomega)$ is synthetic, one can generate unlimited number of random features, which results in a much better approximation to the exact kernel \eqref{DKernel}. In contrast, RSM requires held-out samples from the data, which could be quite limited for a small data set.
Second, in some cases, even with a small or similar number of random features to RSM, the performance of the selected distribution still leads to significantly better results. For those cases we conjecture that the selected $p(\bomega)$ generates objects that capture semantic information more relevant to the estimation of $f(\bx)$, \emph{when coupled} with our feature map under the dissimilarity measure $d(\bx,\bomega)$. 

\section{Analysis}
\label{sec:analysis}

In this section, we analyze the proposed framework from the perspectives of error decomposition. Let $\cH$ be the RKHS corresponding to the kernel \eqref{DKernel}. Let
\begin{equation}\label{risk_minimizer}
f_{C}:=\underset{f\in \cH}{argmin} \E[ \ell( f(\bx), y ) ] \ \ \
s.t.           \|f\|_{\cH}\leq C
\end{equation}
be the population risk minimizer subject to the RKHS norm constraint $\|f\|_{\cH}\leq C$. And let 
\begin{equation}\label{erm_minimizer}
\hf_{n}:=\underset{f\in \cH}{argmin} \frac{1}{n}\sum_{i=1}^n \ell( f(\bx_i), y_i )  \ \ \
s.t.           \|f\|_{\cH}\leq C
\end{equation}
be the corresponding empirical risk minimizer. In addition, let $\tf_R$ be the estimated function from our \emph{random feature approximation} (Algorithm \ref{alg:RF}). Then denote the population and empirical risks as $L(f)$ and $\hL(f)$ respectively. We have the following risk decomposition $L(\tf_R)-L(f)=$
\begin{equation*}
\begin{aligned}
\underbrace{(L(\tf_R)-L(\hf_n))}_{random feature} + \underbrace{(L(\hf_n)-L(f_C))}_{estimation} + \underbrace{(L(f_C)-L(f))}_{approximation}
\end{aligned}
\end{equation*}

In the following, we will discuss the three terms from the rightmost to the leftmost.

\textbf{Function Approximation Error.}
The RKHS implied by the kernel \eqref{DKernel} is 
\begin{equation*}
    \cH:=\left\{f \;\middle|\; f(\bx)=\sum_{j=1}^m \alpha_{j} k(\bx_j,\bx),\; \bx_j\in\X,\forall j\in [m],\; m \in \mathbb{N}\right\},
\end{equation*}
which is a smaller function space than the space of Lipschitz-continuous function w.r.t. the distance $d(\bx_1,\bx_2)$. As we show, any function $f\in\cH$ is Lipschitz-continous w.r.t. the distance $d(.,.)$.
\begin{proposition}\label{thm:RKHS_lip}
Let $\cH$ be the RKHS corresponding to the kernel \eqref{DKernel} derived from some metric $d(.,.)$. For any $f\in\cH$, 
$$
|f(\bx_1)-f(\bx_2)|\leq L_f d(\bx_1,\bx_2)
$$
where $L_f=\gamma C$.
\end{proposition}

While any $f$ in the RKHS is Lipschitz-continuous w.r.t. the given distance $d(.,.)$, we are interested in imposing additional smoothness via the RKHS norm constraint $\|f\|_{\cH}\leq C$, and by the kernel parameter $\gamma$. The hope is that the best function $f_C$ within this class approximates the true function $f$ well in terms of the approximation error
$
L(f_C)-L(f).
$
The stronger assumption made by the RKHS gives us a qualitatively better estimation error, as discussed below.

\textbf{Estimation Error.}
Define $D_{\lambda}$ as
$$
D_{\lambda}:=\sum_{j=1}^{\infty} \frac{1}{1+\lambda/\mu_j}
$$
where $\{\mu_j\}_{j=1}^{\infty}$ is the eigenvalues of the kernel \eqref{DKernel2} and $\lambda$ is a tuning parameter. It holds that for any $\lambda\geq D_{\lambda}/n$, with probability at least $1-\delta$,
$
L(\hf_n)-L(f_C) \leq c(\log\frac{1}{\delta})^2C^2 \lambda
$
for some universal constant $c$ \cite{zhang2005learning}. Here we would like to set $\lambda$ as small as possible (as a function of $n$). By using the following kernel-independent bound:
$
D_{\lambda} \leq 1/\lambda,
$
we have $\lambda=1/\sqrt{n}$ and thus a bound on the estimation error

\begin{equation}\label{est_error}
L(\hf_n)-L(f_C) \leq c(\log\frac{1}{\delta})^2C^2 \sqrt{\frac{1}{n}}.
\end{equation}

The estimation error is quite standard for a RKHS estimator. It has a much better dependency w.r.t. $n$ (i.e. $n^{-1/2}$) compared to that of \emph{$k$-nearest-neighbor method} (i.e. $n^{-2/(2+p_{\X,d})}$) especially for higher effective dimension. A more careful analysis might lead to tighter bound on $D_{\lambda}$ and also a better rate w.r.t. $n$. However, the analysis of $D_{\lambda}$ for our kernel \eqref{DKernel} is much more difficult than that of typical cases as we do not have an analytic form of the kernel. 

\textbf{Random Feature Approximation.}
Denote $\hat L(.)$ as the empirical risk function. The error from RF approximation $L(\tilde f_R)-L(\hat f_n)$ can be further decomposed as 
$$
(L(\tilde f_R)-\hat L(\tilde f_R)) + (\hat L(\tilde f_R)-\hat L(\hat f_n)) + (\hat L(\hat f_n)-L(\hat f_n))
$$
where the first and third terms can be bounded via the same estimation error bound in \eqref{est_error}, as both $\tilde f_R$ and $\hat f_n$ have RKHS norm bounded by $C$. Therefore, in the following, we focus only on the second term of empirical risk. We start by analyzing the approximation error of the kernel
$
\Delta_R(\bx_1,\bx_2)=\tk_R(\bx_1,\bx_2)-k(\bx_1,\bx_2)
$
where
\begin{equation}\label{RF}
\tk_R(\bx_1,\bx_2):=\frac{1}{R}\sum_{j=1}^R \phi_j(\bx_1)\phi_j(\bx_2).
\end{equation}

\begin{proposition}\label{thm:RF}
Let $\Delta_R(\bx_1,\bx_2)=k(\bx_1,\bx_2)-\tk(\bx_1,\bx_2)$, we have uniform convergence of the form
{\small
\begin{equation*}\label{converge_result}
P\left\{ \max_{\bx_1,\bx_2\in\X} |\Delta_R(\bx_1,\bx_2)| > 2t\right\} \leq 2\left(\frac{12\gamma}{t}\right)^{2p_{\X,d}}e^{-Rt^2/2},
\end{equation*}
}
where $p_{\X,d}$ is the effective dimension of $\X$ under metric $d(.,.)$. In other words, to guarantee $|\Delta_R(\bx_1,\bx_2)|\leq \epsilon$ with probability at least $1-\delta$, it suffices to have
$$
R = \Omega\biggl(\frac{p_{\X,d}}{\epsilon^2}\log(\frac{\gamma}{\epsilon})+\frac{1}{\epsilon^2}\log(\frac{1}{\delta}) \biggr).
$$
\end{proposition}

Proposition~\ref{thm:RF} gives an approximation error in terms of kernel evaluation. To get a bound on the empirical risk $\hL(\tf_R)-\hL(\hf_n)$, consider the optimal solution of the empirical risk minimization. By the Representer theorem we have
$
\hf_n(\bx)=\frac{1}{n}\sum_{i}\alpha_i k(\bx_i,\bx)
$
and
$
\tf_R(\bx)=\frac{1}{n}\sum_{i}\talpha_i \tk(\bx_i,\bx).
$
Therefore, we have the following corollary.

\begin{corollary}\label{cor:obj_bound}
To guarantee 
$
\hL(\tf_R)-\hL(\hf_n) \leq \epsilon, 
$
with probability $1-\delta$, it suffices to have
$$
R = \Omega\biggl(\frac{p_{\X,d}M^2A^2}{\epsilon^2}\log(\frac{\gamma}{\epsilon})+\frac{M^2A^2}{\epsilon^2}\log(\frac{1}{\delta}) \biggr).
$$
where $M$ is the Lipschitz-continuous constant of the loss function $\ell(.,y)$, and $A$ is a bound on $\|\balpha\|_1/n$.

\end{corollary}

For most of loss functions, $A$ and $M$ are typically small constants. Therefore, Corollary \ref{cor:obj_bound} states that it suffices to have number of Random Features proportional to the \emph{effective dimension} $O(p_{\X,d}/\epsilon^2)$ to achieve an $\epsilon$ approximation error. 

Combining the three error terms, we can show that the proposed framework can achieve $\epsilon$-suboptimal performance.

\begin{claim}
Let $\tf_R$ be the estimated function from our \emph{random feature approximation} based ERM estimator in Algorithm \ref{alg:RF}, and let $f^*$ denote the desired target function. Suppose further that for some absolute constants $c_1, c_2 > 0$ (up to some logarithmic factor of $1/\epsilon$ and $1/\delta$):
\begin{enumerate}
    \item The target function $f^*$ lies close to the population risk minimizer $f_C$ lying in the RKHS spanned by the D2KE kernel: $L(f_C)-L(f) \le \epsilon/2$.
    \item The number of training samples  $n \ge c_1 \, C^4/\epsilon^2$.
    \item The number of random features $R \ge c_2 p_{\X,d}/\epsilon^2$.
\end{enumerate}
We then have that: $L(\tf_R) - L(f^*) \le \epsilon$ with probability $1-\delta$.
\end{claim}

\section{Experiments}
In this section, we evaluate the proposed method in three different domains involving time-series, strings, and images. Firstly, we discuss the dissimilarity measures and data characteristics for each set of experiments. Then we introduce compared distance-based methods and report their results.

\textbf{Distance Measures.} We have chosen three well-known dissimilarity measures: 1) Dynamic Time Warping (DTW) for time-series \cite{berndt1994using}; 2) Edit Distance (Levenshtein distance) for string \cite{navarro2001guided}; 3) (Modified) Hausdorff distance \cite{huttenlocher1993comparing, dubuisson1994modified} for measuring the closeness of two cloud of points for images. 
Since most distance measures are computationally demanding, having quadratic complexity, we adapted or implemented C-MEX programs for them; other codes were written in Matlab. 

\textbf{Datasets.} For each domain, we selected 4 datasets for our experiments. For time-series data, all are multivariate time-series; three are from the UCI Machine Learning repository \cite{frank2010uci}, the other is generated from the IQ (In-phase and Quadrature components) samples from a wireless line-of-sight communication system from GMU. For string data, the size of alphabet is between 4 and 8; two of them are from the UCI Machine Learning repository and the other two from the LibSVM Data Collection \cite{chang2011libsvm}. All image datasets derived from Kaggle; we  computed a set of SIFT-descriptors to represent each image. We divided each dataset into 70/30 train and test subsets (if there was no predefined train/test split). Properties of these datasets are summarized in Table \ref{tb:info of datasets} in Appendix \ref{App:General Experimental Settings}. 

\textbf{Baselines.} We compare D2KE against KNN, DSK\_RBF \cite{haasdonk2004learning}, DSK\_ND \cite{haasdonk2004learning}, GDK\_LED \cite{pekalska2001generalized}, and RSM \cite{pekalska2001generalized}. Among these baselines, KNN, DSK\_RBF, DSK\_ND, and GDK\_LED have quadratic complexity $O(N^2L^2)$ in both the number of data samples and the length of the sequences, while RSM has computational complexity $O(NRL^2)$, linear in the number of data samples but still quadratic in the length of the sequence. These compare to our method, D2KE, which has complexity $O(NRL)$, linear in both the number of data samples and the length of the sequence. For each method, we search for the best parameters on the training set by performing 10-fold cross validation. For our new method D2KE, since we generate random samples from the distribution, we can use as many as needed to achieve performance close to an exact kernel. We report the best number in the range $R = [4, 4096]$ (typically the larger $R$ is, the better the accuracy). We employ a linear SVM implemented using LIBLINEAR (Fan et al., 2008) for all embedding-based methods (GDK\_LED, RSM, and D2KE) and use LIBSVM \cite{chang2011libsvm} for precomputed dissimilairty kernels (DSK\_RBF and DSK\_ND). More details of experimental setup is provided in Appendix \ref{App:General Experimental Settings}. 

\begin{table*}[th]
\centering
\caption{Classification performance comparison on time-series.} 
\label{tb:comp_time-series}
\scriptsize
\newcommand{\Bd}[1]{\textbf{#1}}
\vspace{-4mm}
\begin{center}
    \begin{tabular}{ c cc cc cc cc cc cc}
    \hline
    \multicolumn{1}{c}{Methods}
    & \multicolumn{2}{c}{D2KE} 
    & \multicolumn{2}{c}{KNN}
    & \multicolumn{2}{c}{DSK\_RBF} 
    & \multicolumn{2}{c}{DSK\_ND} 
    & \multicolumn{2}{c}{GDK\_LED}
    & \multicolumn{2}{c}{RSM} \\ \hline 
	\multicolumn{1}{c}{Datasets} & Accu & Time & Accu & Time & Accu & Time & Accu & Time & Accu & Time & Accu & Time \\ \hline
	Auslan  & 92.60 & 42.4 & 70.26 & 10.0 & 92.47 & 43.4 & 89.74 & 44.6 & \Bd{92.72} & 65.0 & 88.96 & 18.6\\
	pentip & \Bd{99.88} & 1.4 & 98.37 & 27.3 & 98.02 & 125.4 & 70.40 & 126.6 & 97.20 & 13.13 & \Bd{99.88} & 23.6 \\
	ActRecog & \Bd{64.72} & 44.4 & 53.43 & 15.5 & 55.58 & 64.9 & 45.31 & 68.0 & 55.33 & 73.5 & 62.44 & 14.5 \\ 
    IQ\_radio & \Bd{86.87} & 469.3 & 60.25 & 3734 & 77.41 & 13381 & 47.31 & 12251 & 82.17 & 18787 & 70.84 & 575.9 \\ \hline
    \end{tabular}   
\end{center}
\vspace{-4mm}
\end{table*}

\begin{table*}[th]
\centering
\caption{Classification performance comparison on strings.} 
\label{tb:comp_strings}
\scriptsize
\newcommand{\Bd}[1]{\textbf{#1}}
\vspace{-4mm}
\begin{center}
    \begin{tabular}{ c cc cc cc cc cc cc}
    \hline
    \multicolumn{1}{c}{Methods}
    & \multicolumn{2}{c}{D2KE} 
    & \multicolumn{2}{c}{KNN}
    & \multicolumn{2}{c}{DSK\_RBF} 
    & \multicolumn{2}{c}{DSK\_ND} 
    & \multicolumn{2}{c}{GDK\_LED}
    & \multicolumn{2}{c}{RSM} \\ \hline 
	\multicolumn{1}{c}{Datasets} & Accu & Time & Accu & Time & Accu & Time & Accu & Time & Accu & Time & Accu & Time \\ \hline
	bit-str4  & \Bd{90.00} & 3.4 & 80.00 & 1.7 & 88.33 & 3.9 & 86.67 & 3.5 & 83.33 & 1.3 & 86.67 & 2.3\\
	splice & \Bd{90.17} & 46.9 & 79.41 & 63.2 & 87.88 & 204.9 & 85.89 & 208.2 & 85.58 & 111.6 & 86.10 & 47.3 \\
	mnist-str4 & \Bd{98.76} & 3376 & 97.75 & 36840 & 98.66 & 59925 & 91.92 & 59845 & 94.81 & 102130 & 97.86 & 943.5 \\ 
    mnist-str8 & \Bd{98.54} & 4096 & 96.58 & 9207 & 97.5 & 18666 & 92.66 & 18604 & 94.62 & 40498 & 97.61 & 308.6 \\ \hline
    \end{tabular}   
\end{center}
\vspace{-4mm}
\end{table*}

\begin{table*}[th]
\centering
\caption{Classification performance comparison on images.} 
\label{tb:comp_images}
\scriptsize
\newcommand{\Bd}[1]{\textbf{#1}}
\vspace{-4mm}
\begin{center}
    \begin{tabular}{ c cc cc cc cc cc cc}
    \hline
    \multicolumn{1}{c}{Methods}
    & \multicolumn{2}{c}{D2KE} 
    & \multicolumn{2}{c}{KNN}
    & \multicolumn{2}{c}{DSK\_RBF} 
    & \multicolumn{2}{c}{DSK\_ND} 
    & \multicolumn{2}{c}{GDK\_LED}
    & \multicolumn{2}{c}{RSM} \\ \hline 
	\multicolumn{1}{c}{Datasets} & Accu & Time & Accu & Time & Accu & Time & Accu & Time & Accu & Time & Accu & Time \\ \hline
	flower  & 46.03 & 22.0 & 33.33 & 16.4 & 36.51 & 103.5 & 36.51 & 102.4 & \Bd{50.79} & 213.4 & 33.33 & 18.6\\
	decor & 68.76 & 70.3 & 61.81 & 117.3 & 70.83 & 1225.1 & 70.14 & 1221.9 & \Bd{71.52} & 1043.4 & 68.75 & 1625.2 \\
	style & \Bd{40.29} & 20.5 & 36.57 & 48.0 & 38.06 & 450.3 & 30.59 & 449.02 & 38.43 & 397.4 & 37.68 & 652.6 \\ 
    letters2 & 54.05 & 30.5 & 42.52 & 10.9 & 54.55 & 101.9 & 53.27 & 99.7 & \Bd{58.54} & 115.4 & 53.34 & 29.8 \\ \hline
    \end{tabular}   
\end{center}
\vspace{-2mm}
\end{table*}

\textbf{Results.} As shown in Tables \ref{tb:comp_time-series},  \ref{tb:comp_strings}, and \ref{tb:comp_images}, D2KE can consistently outperform or match all other baselines in terms of classification accuracy while requiring far less computation time. There are several observations worth noting here. First, D2KE performs much better than KNN, supporting our claim that D2KE can be a strong alternative to KNN across applications. Second, compared to the two distance substitution kernels DSK\_RBF and DSK\_ND, our method can achieve much better performance, suggesting that a representation induced from a truly p.d. kernel makes significantly better use of the data than indefinite kernels. Among all methods, RSM is closest to our method in terms of practical construction of the feature matrix. However, the random objects (time-series, strings, or sets) sampled by D2KE performs significantly better, as we discussed in section \ref{sec:d2ke}. GDK\_LED performs best in image domain, which may be contributed to both by transductive training and by the SVD operation which allow it to  directly access features of the test set and denoise unwanted information from the raw images. More detailed discussions of the experimental results for each domain are given in Appendix \ref{App:Detailed Experimental Results on Time-Series, Strings, and Images}.

\section{Conclusion and Future Work}
In this work, we propose a general framework for deriving a \emph{positive-definite} kernel and a feature embedding function from a given dissimilarity measure between input objects. The framework is especially useful for structured input domains such as sequences, time-series, and sets, where many well-established dissimilarity measures have been developed. Our framework subsumes a couple of existing approaches as special or limiting cases, and also opens up a new direction for creating embeddings of structured objects based on distance to random objects. A promising future direction is to develop such distance-based embeddings within a deep architecture to handle structured inputs in an end-to-end learning system.


\clearpage
\bibliographystyle{IEEEtran}
\bibliography{IEEEabrv,D2KE_NIPS18}

\begin{thebibliography}{10}
\providecommand{\url}[1]{#1}
\csname url@samestyle\endcsname
\providecommand{\newblock}{\relax}
\providecommand{\bibinfo}[2]{#2}
\providecommand{\BIBentrySTDinterwordspacing}{\spaceskip=0pt\relax}
\providecommand{\BIBentryALTinterwordstretchfactor}{4}
\providecommand{\BIBentryALTinterwordspacing}{\spaceskip=\fontdimen2\font plus
\BIBentryALTinterwordstretchfactor\fontdimen3\font minus
  \fontdimen4\font\relax}
\providecommand{\BIBforeignlanguage}[2]{{%
\expandafter\ifx\csname l@#1\endcsname\relax
\typeout{** WARNING: IEEEtran.bst: No hyphenation pattern has been}%
\typeout{** loaded for the language `#1'. Using the pattern for}%
\typeout{** the default language instead.}%
\else
\language=\csname l@#1\endcsname
\fi
#2}}
\providecommand{\BIBdecl}{\relax}
\BIBdecl

\bibitem{pkkalska2005dissimilarity}
E.~Pkkalska and R.~Duin, ``The dissimilarity representation for pattern
  recognition,'' \emph{World Scientific}, 2005.

\bibitem{duin2012dissimilarity}
R.~P. Duin and E.~P{\k{e}}kalska, ``The dissimilarity space: Bridging
  structural and statistical pattern recognition,'' \emph{Pattern Recognition
  Letters}, vol.~33, no.~7, pp. 826--832, 2012.

\bibitem{balcan2008theory}
M.-F. Balcan, A.~Blum, and N.~Srebro, ``A theory of learning with similarity
  functions,'' \emph{Machine Learning}, vol.~72, no. 1-2, pp. 89--112, 2008.

\bibitem{cortes2012algorithms}
C.~Cortes, M.~Mohri, and A.~Rostamizadeh, ``Algorithms for learning kernels
  based on centered alignment,'' \emph{Journal of Machine Learning Research},
  vol.~13, no. Mar, pp. 795--828, 2012.

\bibitem{scholkopf1999input}
B.~Scholkopf, S.~Mika, C.~J. Burges, P.~Knirsch, K.-R. Muller, G.~Ratsch, and
  A.~J. Smola, ``Input space versus feature space in kernel-based methods,''
  \emph{IEEE transactions on neural networks}, vol.~10, no.~5, pp. 1000--1017,
  1999.

\bibitem{balcan2008discriminative}
M.-F. Balcan, A.~Blum, and S.~Vempala, ``A discriminative framework for
  clustering via similarity functions,'' in \emph{Proceedings of the fortieth
  annual ACM symposium on Theory of computing}.\hskip 1em plus 0.5em minus
  0.4em\relax ACM, 2008, pp. 671--680.

\bibitem{chen2009similarity}
Y.~Chen, E.~K. Garcia, M.~R. Gupta, A.~Rahimi, and L.~Cazzanti,
  ``Similarity-based classification: Concepts and algorithms,'' \emph{Journal
  of Machine Learning Research}, vol.~10, no. Mar, pp. 747--776, 2009.

\bibitem{ong2004learning}
C.~S. Ong, X.~Mary, S.~Canu, and A.~J. Smola, ``Learning with non-positive
  kernels,'' in \emph{Proceedings of the twenty-first international conference
  on Machine learning}.\hskip 1em plus 0.5em minus 0.4em\relax ACM, 2004,
  p.~81.

\bibitem{lin2003study}
H.-T. Lin and C.-J. Lin, ``A study on sigmoid kernels for svm and the training
  of non-psd kernels by smo-type methods,'' \emph{submitted to Neural
  Computation}, vol.~3, pp. 1--32, 2003.

\bibitem{pekalska2001generalized}
E.~Pekalska, P.~Paclik, and R.~P. Duin, ``A generalized kernel approach to
  dissimilarity-based classification,'' \emph{Journal of machine learning
  research}, vol.~2, no. Dec, pp. 175--211, 2001.

\bibitem{chen2008training}
J.~Chen and J.~Ye, ``Training svm with indefinite kernels,'' in
  \emph{Proceedings of the 25th international conference on Machine
  learning}.\hskip 1em plus 0.5em minus 0.4em\relax ACM, 2008, pp. 136--143.

\bibitem{chen2009learning}
Y.~Chen, M.~R. Gupta, and B.~Recht, ``Learning kernels from indefinite
  similarities,'' in \emph{Proceedings of the 26th Annual International
  Conference on Machine Learning}.\hskip 1em plus 0.5em minus 0.4em\relax ACM,
  2009, pp. 145--152.

\bibitem{pekalska2006dissimilarity}
E.~Pekalska and R.~P. Duin, ``Dissimilarity-based classification for vectorial
  representations,'' in \emph{Pattern Recognition, 2006. ICPR 2006. 18th
  International Conference on}, vol.~3.\hskip 1em plus 0.5em minus 0.4em\relax
  IEEE, 2006, pp. 137--140.

\bibitem{pekalska2008beyond}
------, ``Beyond traditional kernels: Classification in two dissimilarity-based
  representation spaces,'' \emph{IEEE Transactions on Systems, Man, and
  Cybernetics, Part C (Applications and Reviews)}, vol.~38, no.~6, pp.
  729--744, 2008.

\bibitem{haasdonk2004learning}
B.~Haasdonk and C.~Bahlmann, ``Learning with distance substitution kernels,''
  in \emph{Joint Pattern Recognition Symposium}.\hskip 1em plus 0.5em minus
  0.4em\relax Springer, 2004, pp. 220--227.

\bibitem{scholkopf2001kernel}
B.~Sch{\"o}lkopf, ``The kernel trick for distances,'' in \emph{Advances in
  neural information processing systems}, 2001, pp. 301--307.

\bibitem{graepel1999classification}
T.~Graepel, R.~Herbrich, P.~Bollmann-Sdorra, and K.~Obermayer, ``Classification
  on pairwise proximity data,'' in \emph{Advances in neural information
  processing systems}, 1999, pp. 438--444.

\bibitem{rahimi2008random}
A.~Rahimi and B.~Recht, ``Random features for large-scale kernel machines,'' in
  \emph{Advances in neural information processing systems}, 2008, pp.
  1177--1184.

\bibitem{dubuisson1994modified}
M.-P. Dubuisson and A.~K. Jain, ``A modified hausdorff distance for object
  matching,'' in \emph{Pattern Recognition, 1994. Vol. 1-Conference A: Computer
  Vision \& Image Processing., Proceedings of the 12th IAPR International
  Conference on}, vol.~1.\hskip 1em plus 0.5em minus 0.4em\relax IEEE, 1994,
  pp. 566--568.

\bibitem{gyorfi2006distribution}
L.~Gy{\"o}rfi, M.~Kohler, A.~Krzyzak, and H.~Walk, \emph{A distribution-free
  theory of nonparametric regression}.\hskip 1em plus 0.5em minus 0.4em\relax
  Springer Science \& Business Media, 2006.

\bibitem{collins2002convolution}
M.~Collins and N.~Duffy, ``Convolution kernels for natural language,'' in
  \emph{Advances in neural information processing systems}, 2002, pp. 625--632.

\bibitem{cuturi2011fast}
M.~Cuturi, ``Fast global alignment kernels,'' in \emph{Proceedings of the 28th
  international conference on machine learning (ICML-11)}, 2011, pp. 929--936.

\bibitem{wu2016revisiting}
L.~Wu, I.~E. Yen, J.~Chen, and R.~Yan, ``Revisiting random binning features:
  Fast convergence and strong parallelizability,'' in \emph{Proceedings of the
  22nd ACM SIGKDD International Conference on Knowledge Discovery and Data
  Mining}.\hskip 1em plus 0.5em minus 0.4em\relax ACM, 2016, pp. 1265--1274.

\bibitem{chen2016efficient}
J.~Chen, L.~Wu, K.~Audhkhasi, B.~Kingsbury, and B.~Ramabhadrari, ``Efficient
  one-vs-one kernel ridge regression for speech recognition,'' in
  \emph{Acoustics, Speech and Signal Processing (ICASSP), 2016 IEEE
  International Conference on}.\hskip 1em plus 0.5em minus 0.4em\relax IEEE,
  2016, pp. 2454--2458.

\bibitem{kar2012random}
P.~Kar and H.~Karnick, ``Random feature maps for dot product kernels,'' in
  \emph{Artificial Intelligence and Statistics}, 2012, pp. 583--591.

\bibitem{bach2017equivalence}
F.~Bach, ``On the equivalence between kernel quadrature rules and random
  feature expansions,'' \emph{Journal of Machine Learning Research}, vol.~18,
  no.~21, pp. 1--38, 2017.

\bibitem{sinha2016learning}
A.~Sinha and J.~C. Duchi, ``Learning kernels with random features,'' in
  \emph{Advances in Neural Information Processing Systems}, 2016, pp.
  1298--1306.

\bibitem{zhang2005learning}
T.~Zhang, ``Learning bounds for kernel regression using effective data
  dimensionality,'' \emph{Neural Computation}, vol.~17, no.~9, pp. 2077--2098,
  2005.

\bibitem{berndt1994using}
D.~J. Berndt and J.~Clifford, ``Using dynamic time warping to find patterns in
  time series.'' in \emph{KDD workshop}, vol.~10, no.~16.\hskip 1em plus 0.5em
  minus 0.4em\relax Seattle, WA, 1994, pp. 359--370.

\bibitem{navarro2001guided}
G.~Navarro, ``A guided tour to approximate string matching,'' \emph{ACM
  computing surveys (CSUR)}, vol.~33, no.~1, pp. 31--88, 2001.

\bibitem{huttenlocher1993comparing}
D.~P. Huttenlocher, G.~A. Klanderman, and W.~J. Rucklidge, ``Comparing images
  using the hausdorff distance,'' \emph{IEEE Transactions on pattern analysis
  and machine intelligence}, vol.~15, no.~9, pp. 850--863, 1993.

\bibitem{frank2010uci}
A.~Frank and A.~Asuncion, ``Uci machine learning repository [http://archive.
  ics. uci. edu/ml]. irvine, ca: University of california,'' \emph{School of
  information and computer science}, vol. 213, 2010.

\bibitem{chang2011libsvm}
C.-C. Chang and C.-J. Lin, ``Libsvm: a library for support vector machines,''
  \emph{ACM transactions on intelligent systems and technology (TIST)}, vol.~2,
  no.~3, p.~27, 2011.

\bibitem{wu2015preconditioned}
L.~Wu and A.~Stathopoulos, ``A preconditioned hybrid svd method for accurately
  computing singular triplets of large matrices,'' \emph{SIAM Journal on
  Scientific Computing}, vol.~37, no.~5, pp. S365--S388, 2015.

\bibitem{wu2017primme_svds}
L.~Wu, E.~Romero, and A.~Stathopoulos, ``Primme\_svds: A high-performance
  preconditioned svd solver for accurate large-scale computations,'' \emph{SIAM
  Journal on Scientific Computing}, vol.~39, no.~5, pp. S248--S271, 2017.

\bibitem{sezgin2004survey}
M.~Sezgin and B.~Sankur, ``Survey over image thresholding techniques and
  quantitative performance evaluation,'' \emph{Journal of Electronic imaging},
  vol.~13, no.~1, pp. 146--166, 2004.

\bibitem{gao20123}
Y.~Gao, M.~Wang, D.~Tao, R.~Ji, and Q.~Dai, ``3-d object retrieval and
  recognition with hypergraph analysis,'' \emph{IEEE Transactions on Image
  Processing}, vol.~21, no.~9, pp. 4290--4303, 2012.

\end{thebibliography}

\clearpage
\appendix

\section{Proof of Theorem \ref{thm:RKHS_lip} and Theorem \ref{thm:RF}}

\subsection{Proof of Theorem \ref{thm:RKHS_lip}}

\begin{proof}
Note the function $g(t)=exp(-\gamma t)$ is Lipschitz-continuous with Lipschitz constant $\gamma$. Therefore,
\begin{align*}
&|f(\bx_1)-f(\bx_2)|=|\langle f, \phi(\bx_1)-\phi(\bx_2)\rangle|\\
&\leq \|f\|_{\cH} \|\phi(\bx_1)-\phi(\bx_2)\|_{\cH}\\
&= \|f\|_{\cH} \sqrt{ \int_{\bomega} p(\bomega) ( \phi_{\bomega}(\bx_1)-\phi_{\bomega}(\bx_2))^2 d\bomega  }\\
&\leq  \|f\|_{\cH} \sqrt{ \int_{\bomega} p(\bomega) \gamma^2 |d(\bx_1,\bomega)-d(\bx_2,\bomega)|^2 d\bomega  }\\
&\leq \gamma \|f\|_{\cH}  \sqrt{ \int_{\bomega} p(\bomega)  d(\bx_1,\bx_2)^2 d\bomega  }\\
&\leq \gamma  \|f\|_{\cH}  d(\bx_1,\bx_2)\leq \gamma  C  d(\bx_1,\bx_2)
\end{align*}
\end{proof}

\subsection{Proof of Theorem \ref{thm:RF}}

\begin{proof}
Our goal is to bound the magnitude of $\Delta_R(\bx_1,\bx_2)=\tk_R(\bx_1,\bx_2)-k(\bx_1,\bx_2)$. Since $E[\Delta_R(\bx_1,\bx_2)]=0$ and $|\Delta_R(\bx_1,\bx_2)|\leq 1$, from Hoefding's inequality, we have
$$
P\left\{ |\Delta_R(\bx_1,\bx_2)|\geq t \right\} \leq 2 \exp(-Rt^2/2)
$$
 a given input pair $(\bx_1,\bx_2)$. To get a unim bound that holds  $\forall (\bx_1,\bx_2)\in\X\times\X$, we find an $\epsilon$-covering $\cE$ of $\X$ w.r.t. $d(.,.)$ of size $N(\epsilon,\X,d)$. Applying union bound over the $\epsilon$-covering $\cE$ for $\bx_1$ and $\bx_2$, we have
\begin{equation}\label{tmp1}
P\left\{ \max_{\bx_1'\in\cE,\bx_2'\in\cE} |\Delta_R(\bx_1',\bx_2')| > t \right\} \leq 2|\cE|^2\exp(-Rt^2/2).
\end{equation}
Then by the definition of $\cE$ we have $|d(\bx_1,\bomega)-d(\bx_1',\bomega)|\leq d(\bx_1,\bx_1')\leq \epsilon$. Together with the fact that $\exp(-\gamma t)$ is Lipschitz-continuous with parameter $\gamma$ for $t\geq0$, we have 
$$
|\phi_{\bomega}(\bx_1)-\phi_{\bomega}(\bx_1')|\leq \gamma\epsilon
$$
and thus
$$
|\tk_R(\bx_1,\bx_2)-\tk_R(\bx_1',\bx_2')|\leq 3\gamma\epsilon, 
$$
$$
|k(\bx_1,\bx_2)-k(\bx_1',\bx_2')|\leq 3\gamma\epsilon
$$
for $\gamma\epsilon$ chosen to be $\leq 1$. This gives us 
\begin{equation}\label{tmp2}
|\Delta_R(\bx_1,\bx_2)-\Delta_R(\bx_1',\bx_2')|\leq 6\gamma\epsilon
\end{equation}
Combining \eqref{tmp1} and \eqref{tmp2}, we have
\begin{equation}\label{tmp3}
\begin{aligned}
&P\left\{ \max_{\bx_1'\in\cE,\bx_2'\in\cE} |\Delta_R(\bx_1',\bx_2')| > t + 6\gamma \epsilon\right\} \\ 
&\leq 2\left(\frac{2}{\epsilon}\right)^{2p_{\X,d}}\exp(-Rt^2/2).
\end{aligned}
\end{equation}
Choosing $\epsilon=t/6\gamma$ yields the result.
\end{proof}

\subsection{Proof for Corollary \ref{cor:obj_bound}}

\begin{proof}
First of all, we have
\begin{align*}
&\frac{1}{n}\sum_{i=1}^n \ell(\frac{1}{n}\sum_{j=1}^n \talpha_j\tk(\bx_j,\bx_i),y_i)\\
&\leq \frac{1}{n}\sum_{i=1}^n \ell(\frac{1}{n}\sum_{j=1}^n \alpha_j\tk(\bx_j,\bx_i),y_i)
\end{align*}
by the optimality of $\{\talpha_j\}_{j=1}^n$ w.r.t. the objective using the approximate kernel. Then we have
\begin{align*}
&\hL(\tf_R)-\hL(\hf_n)\\ 
&\leq \frac{1}{n}\sum_{i=1}^n \ell(\frac{1}{n}\sum_{j=1}^n\alpha_j\tk(\bx_j,\bx_i),y_i)-\ell(\frac{1}{n}\sum_{j=1}^n\alpha_j k(\bx_j,\bx_i),y_i) \\
&\leq M \frac{\|\balpha\|_1}{n}\left(\max_{\bx_1,\bx_2\in\X}|\tk(\bx_1,\bx_2)-k(\bx_1,\bx_2)|\right) \\
&\leq MA\left(\max_{\bx_1,\bx_2\in\X}|\tk(\bx_1,\bx_2)-k(\bx_1,\bx_2)|\right) 
\end{align*}
where $A$ is a bound on $\|\balpha\|_1/n$.
Therefore to guarantee 
$$
\hL(\tf_R)-\hL(\hf_n) \leq \epsilon
$$
we would need $\left(\max_{i,j\in[n]} |\Delta_R(\bx_1,\bx_2)|\right)\leq \hat{\epsilon}:=\epsilon/MA$. Then applying Theorem \ref{thm:RF} leads to the result.
\end{proof}

\section{General Experimental Settings}
\label{App:General Experimental Settings}

\textbf{Baselines.} We compare with the following methods: \\
\textbf{KNN}: a simple yet universal method to apply any distance measure to classification tasks. \\
\textbf{DSK\_RBF} \cite{haasdonk2004learning}: distance substitution kernels, a general framework for kernel construction by substituting  a problem specific distance measure for the Euclidean distance used in ordinary kernel functions. We use a Gaussian RBF kernel. \\
\textbf{DSK\_ND} \cite{haasdonk2004learning}: another class of distance substitution kernels with negative distance. \\
\textbf{GDK\_LED} \cite{pekalska2001generalized}: learning a pseudo-Euclidean linear embedding from the dissimilarity matrix followed by performing singular value decomposition \cite{wu2015preconditioned,wu2017primme_svds}. \\
\textbf{RSM} \cite{pekalska2001generalized}:  building an embedding by computing distances from randomly selected representative samples. 

Among these baselines, KNN, DSK\_RBF, DSK\_ND, and GDK\_LED have quadratic complexity $O(N^2L^2)$ in both the number of data samples and the length of the sequences, while RSM has computational complexity $O(NRL^2)$, linear in the number of data samples but still quadratic in the length of the sequence. These compare to our method, D2KE, which has complexity $O(NRL)$, linear in both the number of data samples and the length of the sequence. 

\textbf{General Setup.} For each method, we search for the best parameters on the training set by performing 10-fold cross validation. Following \cite{haasdonk2004learning}, we use an exact RBF kernel for DSK\_RBF while choosing squared distance for DSK\_ND. Since there no clear indication how many singular vectors should be computed for the GDK\_LED method after construction of the dissimilarity matrix, we compute $R = [4, 512]$ singular vectors and report the best performance. Importantly, we also perform SVD transductively on both train and test data for GDK\_LED; we will show below that this is beneficial. Similarly, we adopted a simple method -- random selection -- to obtain $R = [4, 512]$ data samples as the representative set for GDK\_LW. For our new method D2KE, since we generate random samples from the distribution, we can use as many as needed to achieve performance close to an exact kernel. We report the best number in the range $R = [4, 4096]$ (typically the larger $R$ is, the better the accuracy). We employ a linear SVM implemented using LIBLINEAR (Fan et al., 2008) for all embedding-based methods (GDK\_LED, RSM, and D2KE) and use LIBSVM \cite{chang2011libsvm} for precomputed dissimilairty kernels (DSK\_RBF and DSK\_ND). 

All computations were carried out on a DELL dual-socket system with Intel Xeon processors at 2.93GHz for a total of 16 cores and 250 GB of memory, running the SUSE Linux operating system. To accelerate the computation of all methods, we used multithreading with 12 threads total for various distance computations in all experiments.

\begin{table}[htbp]
\centering
\small
\caption{Properties of the datasets. TS, Str and Img stand for Time-Series, String, and Image respectively. Var/Alpb stands for the number of variables for time-series or image SIFT-descriptors, and for the size of the alphabet for strings. } 
\label{tb:info of datasets}
\vspace{-2mm}
\begin{center}
    \begin{tabular}{ c c c c c c c }
    \hline
    Domain & Name & Var/Alpb & Classes & Train & Test & length \\ \hline 
    TS & Auslan      & 22   & 95 & 1795 & 770 & 45-136  \\
    TS & pentip      & 3    & 20 & 2000 & 858	& 109-205  \\
    TS & ActRecog    & 3    & 7 & 1837 & 788 & 2-151 \\
    TS & IQ\_radio   & 4    & 5 & 6715 & 6715	& 512 \\ \hline
    Str & bit-str4   & 4    & 10 & 140 & 60	& 44/158  \\  
    Str & splice     & 4    & 3 & 2233 & 957 & 60  \\
    Str & mnist-str4 & 4    & 10 & 60000 & 10000 & 34/198 \\ 
    Str & mnist-str8 & 8    & 10 & 60000 & 10000 & 17/99 \\ \hline
    Img & flower     & 128  & 10 & 147 & 63 & 66/429 \\ 
    Img & decor      & 128  & 7 & 340 & 144 & 35/914 \\ 
    Img & style      & 128  & 7 & 625 &	268 & 6/530  \\
    Img & letters2 	 & 128  & 33 & 3277 & 1404 & 1/22 \\   \hline
    \end{tabular}
\end{center}
\vspace{-4mm}
\end{table}

\section{Detailed Experimental Results on Time-Series, Strings, and Images}
\label{App:Detailed Experimental Results on Time-Series, Strings, and Images}

\subsection{Results on multivariate time-series }
\label{App:Results on multivariate time-series}
\textbf{Setup.} For time-series data, we employed the most successful distance measure - DTW - for all methods. For all datasets, a Gaussian distribution was found to be applicable, parameterized by its bandwidth $\sigma$. The best values for $\sigma$ and for the length of random time series were searched in the ranges \text{[1e-3 1e3]} and \text{[2 50]}, respectively. 

\textbf{Results.} As shown in Table \ref{tb:comp_time-series}, D2KE can consistently outperform or match all other baselines in terms of classification accuracy while requiring far less computation time for multivariate time-series. The first interesting observation is that  our method performs substantially better than KNN, often by a large margin, i.e., D2KE achieves 26.62\% higher performance than KNN on IQ\_radio. This is because KNN is sensitive to the data noise common in real-world applications like IQ\_radio, and has notoriously poor performance for high-dimensional data sets like Auslan. Moreover, compared to the two distance substitution kernels DSK\_RBF and DSK\_ND, our method can achieve much better performance, suggesting that a representation induced from a truly p.d. kernel makes significantly better use of the data than indefinite kernels. However, GDK\_LED slightly outperforms D2KE on Auslan, probably due to the embedding matrix (singular vectors) being computed transductively on both train and test data. Among all methods, RSM is closest to our method in terms of practical construction of the feature matrix. However, the random time series sampled by D2KE performs significantly better, as we discussed in section \ref{sec:d2ke}. 

\subsection{Results on strings }
\label{App:Results on strings}
\textbf{Setup.} For string data, there are various well-known edit distances. Here, we choose Levenshtein distance as our distance measure since it can capture global alignments of the underlying strings. We first compute the alphabet from the original data and then uniformly sample characters from this alphabet to generate random strings. We search for the best parameters for  $\gamma$ in the range \text{[1e-5 1]}, and for the length of random strings in the range \text{[2 50]}, respectively.  

\textbf{Results.} As shown in Table \ref{tb:comp_strings}, D2KE consistently performs better than or similarly to other distance-based baselines. Unlike the previous experiments where DTW is not a distance metric, Levenshtein distance is indeed a distance metric; this helps improve the performance of our baselines. However, D2KE still offers a clear advantage over baseline. It is interesting to note that the performance of DSK\_RBF is quite close to our method's, which may be due to DKS\_RBF with Levenshtein distance producing a c.p.d. kernel which can essentially be converted into a p.d. kernel. Notice that on relatively large datasets, our method, D2KE, can achieve better performance, and often with far less computation than other baselines with quadratic complexity in both number and length of data samples. For instance, on mnist-str4 D2KE obtains higher accuracy with an order of magnitude less runtime compared to DSK\_RBF and DSK\_ND, and two orders of magnitude less than GDK\_LED, due to higher computational costs both for kernel matrix construction and for eigendecomposition.

\subsection{Results on Sets of SIFT-descriptors for images }
\label{{App:Results on Sets of SIFT-descriptors for images }}
\textbf{Setup.} For image data, following \cite{pekalska2001generalized,haasdonk2004learning} we use the modified Hausdorff distance (MHD) \cite{dubuisson1994modified} as our distance measure between images, since this distance has shown excellent performance in the literature \cite{sezgin2004survey, gao20123}. We first applied the open-source OpenCV library to generate a sequence of SIFT-descriptors with dimension 128, then MHD to compute the distance between sets of SIFT-descriptors. We generate random images of each SIFT-descriptor uniformly sampled from the unit sphere of the embedding vector space $\R^{128}$. We search for the best parameters for  $\gamma$ in the range \text{[1e-3 1e1]}, and for length of random SIFT-descriptor sequence in the range \text{[3 15]}.

\textbf{Results.} As shown in Table \ref{tb:comp_images}, D2KE performance is near other baselines in most cases. First, GDK\_LED performs best in three cases, which may be contributed to both by transductive training and by the SVD operation which allow it to  directly access features of the test set and denoise unwanted information from the raw images. Nevertheless, the quadratic complexity of GDK\_LED in terms of both the number of images and the length of SIFT descriptor sequences makes it hard to scale to large data. Interestingly, D2KE still performs much better than KNN, again supports our claim that D2KE can be a strong alternative to KNN across applications.

%
%
%

\end{document}